\documentclass[]{interact}

\usepackage{epstopdf}% To incorporate .eps illustrations using PDFLaTeX, etc.
\usepackage[caption=false]{subfig}% Support for small, `sub' figures and tables

\usepackage{color}
\usepackage[longnamesfirst,sort]{natbib}% Citation support using natbib.sty
\bibpunct[, ]{(}{)}{;}{a}{,}{,}% Citation support using natbib.sty
\renewcommand\bibfont{\fontsize{10}{12}\selectfont}% To set the list of references in 10 point font using natbib.sty

\theoremstyle{plain}% Theorem-like structures provided by amsthm.sty
\newtheorem{theorem}{Theorem}[section]
\newtheorem{lemma}[theorem]{Lemma}

\newtheorem{proposition}[theorem]{Proposition}

\theoremstyle{definition}
\newtheorem{definition}[theorem]{Definition}

\theoremstyle{remark}
\newtheorem{remark}{Remark}

\newcommand{\FBR}{FB(MEL, \!RPL)}
\newcommand{\FB}{FB(MEL, \!\L)}

\newcommand{\red}[1]{{\color{red} #1}}
\newcommand{\blue}[1]{{\color{blue} #1}}

\begin{document}

%\articletype{ARTICLE TEMPLATE}% Specify the article type or omit as appropriate

\title{An elementary belief function logic}

\author{
\name{D. Dubois\textsuperscript{a}\thanks{CONTACT D. Dubois, H. Prade, Email:\{didier.dubois,henri.prade\}@irit.fr; L. Godo, Email: godo@iiia.csic.es} and Lluis Godo\textsuperscript{b} and Henri Prade\textsuperscript{a}}
\affil{\textsuperscript{a}IRIT-CNRS, Universit\'e Paul Sabatier, Toulouse, France; \textsuperscript{b}IIIA-CSIC, Bellaterra, Spain}
}

\maketitle

\begin{abstract}
Non-additive uncertainty theories, typically possibility theory, belief functions and imprecise probabilities share a common feature with modal logic: the duality properties between possibility and necessity measures, belief and plausibility functions as well as between upper and lower probabilities extend the duality between possibility and necessity modalities to the graded environment.  It has been shown that the all-or-nothing version of possibility theory can be exactly captured by a minimal epistemic logic (MEL) that uses a very small fragment of the KD modal logic, without resorting to relational semantics. Besides, the  case of belief functions has been studied independently, and a belief function logic has been obtained by extending the modal logic S5 to graded modalities using {\L}ukasiewicz logic, albeit using relational semantics. This paper shows that a simpler belief function logic can be devised by adding {\L}ukasiewicz logic on top of MEL. It allows for a more natural semantics in terms of Shafer basic probability assignments.
\end{abstract}


\section{Introduction}\label{intro}
There are two distinct lines of research that aim at modeling belief and knowledge: modal logic and uncertainty theories. Modal logic extends classical logic by introducing knowledge or belief at the syntactic level using a specific symbol, often denoted by $\Box$, that prefixes logical propositions \citep{Hintikka}. To say that $\Box \varphi$ is true is equated to the claim that proposition $\varphi$ is known or believed. Such a proposition as $\Box \varphi$  is called epistemic. Note that in this approach belief and knowledge are all-or-nothing concepts, i.e., such things as degrees of belief, let alone knowledge, are {not} considered  {(up to a few exceptions, e.g.,  \citep{vanderHoek92}, where grades are generally based on a count of possible worlds rather than referring to a graded scale)}. Moreover, in this approach knowledge is understood as true belief (assuming $\Box \varphi \to \varphi$ is an axiom, called T).

 The second line of research, on the contrary, sees the notion of graded belief as central, and attaches degrees of belief to propositions representing subsets of possible states of affairs. In other words, if $\Omega$ is the set of states of affairs, belief is represented by a set function $g:2^\Omega\to [0, 1]$, where $g(A)$ represents the amount of confidence in a statement of the form $x\in A$, where $x$ is the ill-known state of affairs. Typically, $g(A)$ is a degree of probability. However in the last 50 years, other types of set functions have been used, such as possibility and necessity measures  \citep{Zad78a},  \citep{DP88a, DLP94, DP14}, belief functions \citep{Dempster},  \citep{Shafer},  \citep{Smets-1} or imprecise probabilities \citep{Walley}. The rise of uncertainty theories   different from and often more general than probability functions was basically motivated by the need for a sound representation of partial knowledge (including uncertainty due to sheer ignorance), that probability functions fail to capture. In this paper, we focus on Shafer belief functions, showing that they are natural graded extensions of the all-or-nothing kind of uncertainty captured by epistemic or doxastic modal logics, albeit with simpler semantics based on possibility theory. 
 
 Quite early there have been some scholars pointing out the similarity between Shafer belief functions and modal epistemic logics. The first work along this line seems to be due to  \citet{Ruspini1, Ruspini2}. He pointed out that the degree of belief $Bel(A)= \sum_{E \subseteq A} m(E)$, where $m$ is the underlying mass assigment function, a probability distribution over the set of subsets of $\Omega$, can be viewed as the probability of the support of a  proposition \red{$\varphi$} where $A= Mod(\varphi)$ is the set of models of $\varphi$. In Ruspini's terminology, the support set of $\varphi$ is the set of propositions that imply $\varphi$; a set $E$ such that $m(E) > 0$ (focal, in the sense of Shafer) is called an epistemic set by Ruspini. His work suggested that a belief function logic can be built over the epistemic logic S5.  \citet{Pearl,P88} noticed that the degree of belief can indeed be viewed as a probability of provability, viewing a non-empty subset $E$ such that $E\subseteq A$ as an argument proving $\varphi$ with probability $m(E)$.
 
 \citet{Smets-1} noticed that $Bel(A)$ is of the form $P(\Box \varphi)$, i.e.\ the probability of fully believing $\varphi$, and developed this point in  \citep{Smets-2}. {This line of thought has been also investigated by \citet{Provan89a,Provan89b,Provan90} where the combination of belief functions using Dempster's rule of combination corresponds to a combination of corresponding support clauses. Besides, an early investigation of the propagation of belief degrees in a Boolean logic setting can be found in  \citep{ChatalicDP87}. \citet{Saffiotti92} introduced a so-called belief function logic (BFL) where a first order sentence can be associated with a pair of grades respectively representing the beliefs that the sentence is true and that the sentence is false (hence the sum of these grades should be less than 1).  \citet{Ch16} proposes a propositional reading of evidence theory where the frame of discernment is related to a propositional language, which enables her to consider non-mutually exclusive alternatives in the frame.
{Building on the probability of provability view, \citet{BesKoh} have generalized the concept of belief function, constructing an evidence theory on top of very general logics defined by consequence relations in the sense of Tarski.} {However, as negation need not be defined in these general logics, the usual duality relations between belief  and plausibility functions of Dempster-Shafer theory do not hold in general.}
\citet{Sossai99,Sossai00} proposed a logic for belief functions (also called BFL!) that is a proper extension of classical logic, the axiomatization of which includes both probabilistic and possibilistic semantics as special cases. In this logic, one can represent for any Boolean event, the ``meta-event'' stating that its belief is at least $\alpha$; we can also encode an unnormalized  version of Dempster rule of combination, where the contradiction can receive a positive mass. 
 This logic has been applied convincingly to robotics  \citep{SossaiBCT99,SossaiBC01}.} However, the above-mentioned formalisms have limited expressive power, because they do not account for the modal feature of belief functions in their syntax. In constrast, \citet{GHE01,GHE03} defined a logic for belief functions as a probabilistic logic expressed over a modal logic, namely, the whole modal logic S5. Finally let us mention, for the sake of completeness, quite a different approach to reasoning about belief functions (and other quantitative representations of uncertainty) developed in  \citet{Halpern03}'s book. It is based on a rich language built over linear combinations of likelihood terms that can encode the characteristic property of total monotonicity of belief functions as a denumerable set of axioms.
 
 %and that the underlying  such that $Bel(A)$, could also be expressed as the probability of a certain 
 
 There are two issues when trying to unify the two traditions about belief representation (modalities and set functions). 
 
 First at the syntactic level, what is the most appropriate language for supporting a belief  function logic? Choosing the language of S5, objective and epistemic propositions cohabit and can be combined. However, using the set function approach, all propositions (events) are encapsulated by the set function since all of them are assigned degrees of confidence. So, to design a minimal modal counterpart of belief functions, only epistemic formulas should be used. Besides, complex propositions involving nested modalities belong to the language of S5. Even if they can be simplified using S5 axioms, there is no set function counterparts of propositions prefixed by nested modalities (unless we consider the case of 2nd order probabilities and generalize them accordingly, which leads to formal and conceptual difficulties). 
 
 Second, at the semantic level, do we need accessibility relations to model uncertainty? Most presentations of modal logics, including epistemic logic, equip them with Kripke semantics, that is, accessibility relations. Modal logics at large are indeed tailored for a logical description of relations and their compositions. In the case of the logic S5, they are equivalence relations. The relational semantics makes it possible to evaluate all formulas of the modal language on objective states of affairs. Typically, $\Box \varphi$ is true in state $w\in \Omega$ if $\varphi$ is true in all states $R(w)$ accessible from $w$ via the relation $R$ on interpretations (for S5,  $R(w)$ is the equivalence class of $w$). This has led epistemic logic specialists to claim that imprecise information is expressed by indiscernible states of affairs, which is not obvious to grasp.\footnote{See \cite{BDGP18} for a detailed critical discussion of this view.} In contrast, interpreting  the truth of $\Box \varphi$ as $Bel(Mod(\varphi)) = 1$, we can see that this is the case if $\varphi$ is true in all epistemic (focal) sets where $\varphi$ is true. It clearly suggests that in order to get a simpler belief function logic, we should replace accessibility relations by epistemic sets in the sense of Ruspini. 

The process of simplifying the modal language and semantics has been carried out in the setting of possibility theory with the logic MEL \citep{BD13}. The idea was\begin{itemize}
  \item to restrict the language to epistemic formulas of the form $\Box \varphi$, where $\varphi$ is propositional, and their combination using negation and conjunction; 
  \item to evaluate modal formulas on epistemic sets;
  \item to keep the axioms K, D and a form of necessitation axiom
\end{itemize}  
In such a logic,  we have that $\Box \varphi$ is true in the epistemic set $E$ if and only if $E\subseteq Mod(\varphi)$
if and only if $N_E(Mod(\varphi))= 1$, where $N_E$ is a necessity measure (a special case of belief function {when focal sets are nested}) induced when the epistemic state of the agent is $E$. Noticing that $Bel(A) = \sum_{E\in \Omega} N_E(A)m(E)$, it is clear that we can envisage to build a genuine belief function logic on MEL rather than the full-fledged S5, and evaluate modal formulas on mass assignments. Moreover, previous works on \L ukasiewicz logic make it possible to consider that even if $\varphi$ is a Boolean proposition, $\Box \varphi$ can be a many-valued one ranging on $[0, 1]$, assuming that the degree of belief in $\varphi$ is the degree of truth of $\Box \varphi$.

This is the program followed by this paper whose structure is as follows. Section \ref{sec-MEL} recalls the minimal epistemic logic MEL. Section \ref{Luk-logic} recalls {\L}ukasiewicz fuzzy logic. Section \ref{P-on-MEL} puts probabilities on top of MEL formulas and shows its connection to belief functions.  Section \ref{BF-log} casts the probabilistic version of MEL into \L ukasiewicz logic, thus providing a simpler belief function logic, with simpler syntax and clearer semantics not needing accessibility relations. Finally, Section \ref{truthcons} extends the language to truth-constants in order to reason with quantified beliefs. The conclusion discusses limitations of this logic and future possible developments.

\section{Background I: the logic MEL} \label{sec-MEL}

The usual truth values {\em true} (1) and {\em false} (0) assigned to propositions are of ontological nature (which means
that they are part of the definition of what we call \emph{proposition}), whereas assigning to a proposition a value whose meaning is expressed by the word \emph{unknown} 
sounds like having an epistemic nature: it reveals a knowledge state according to which the truth value of a proposition (in the usual Boolean sense) in a given situation is out of reach (for instance one cannot compute it, either by lack of computing power, or due to a sheer lack of information). It corresponds to an epistemic state for an agent that can neither assert the truth of a Boolean proposition nor its falsity. 

Admitting that the concept of ``unknown'' refers to a knowledge state rather than to an ontic truth value, we may start with Boolean logic where asserted formulas are interpreted as beliefs, and add to its syntax the capability of stating that we ignore the truth value (1 or 0) of propositions. To this end we need to make a syntactic difference between not knowing the truth of a statement in classical propositional logic (CPL) and knowing that this statement is false. The natural framework to achieve this purpose is modal logic, and in particular, the logic KD.  Nevertheless, only a very limited fragment of this language is needed. We do not need nested modalities, as long as we do not need to model introspection (`I believe that I believe..."), nor objective sentences (without modalities) since we only deal with beliefs. The logic MEL \citep{BD09,BD13} was defined for that purpose. 
 
 Given a standard propositional language $\cal L$, consider another propositional language ${\cal L}_\Box$ whose set of propositional variables is of the form ${\cal V}_\Box= \{\Box \varphi \mid \varphi \in {\cal L}\}$ to which the classical connectives $(\land, \lor, \neg, \to, \equiv)$ can be applied. It is endowed with a modality operator expressing certainty, that encapsulates formulas in $\cal L$. So, there is one propositional variable for each formula in $\cal L$. As usual, $\Diamond \varphi$ is short for $\neg \Box \neg \varphi$, and expresses the idea that $\varphi$ cannot be ruled out. {Formulas in ${\cal L}_\Box$ are clearly modal formulas of depth 1, denoted by $\Phi, \Psi, \dots$}. {In particular, the formula $\Diamond \varphi \land \Diamond \neg \varphi$ expresses that the truth-value of $\varphi$ is unknown. The language of MEL could be as well named $\Box$CPL, to highlight the fact that it only involves boxed propositional formulas.
 
MEL is a propositional logic on the language ${\cal L}_\Box$ with the following semantics. {Let $\Omega$ be the set of classical interpretations for the propositional language $\cal L$.
%, i.e., $\Omega$ consists of the set of mappings $w: {\cal L} \to \{0, 1\}$ conforming to the rules of classical propositional logic. 
A model of a propositional formula $\varphi \in {\cal L}$ is an element of $\Omega$;  we will denote by $Mod(\varphi)\subseteq \Omega$ the set of models of $\varphi$.}
%$w \in \Omega$ such that $w(\varphi) = 1$. 
In contrast, models (or interpretations) for MEL correspond to consistent epistemic states, which are simply subsets $\emptyset \neq E \subseteq \Omega$. The truth-evaluation rules of formulas of ${\cal L}_\Box$ in a given epistemic model $E$ are defined as follows:

\begin{itemize}
\item $E \models \Box \varphi$ \; if \; $E \subseteq Mod(\varphi)$
\item $E \models \neg \Phi$ \; if \; $E \not \models \Phi$
\item $E \models \Phi \land \Psi$ \; if \; $E \models \Phi$ and $E \models \Psi$
\end{itemize}
The intuition is that if the epistemic state of an agent is $E$ and $E \subseteq Mod(\varphi)$, then this agent believes that $\varphi$ is true; {in other words, $\varphi$ is true in all situations the agent believe to be possible.} Note that contrary to what is usual in modal logic, modal formulas are not evaluated on 
particular interpretations of the { language} $\mathcal{L}$ because modal formulas in MEL do not refer to the actual world. 

The notion of logical consequence is defined as usual: Let $\Gamma $ be a set of ${\cal L}_\Box$-formulas and $\Phi$ be another such formula; then $\Gamma \models \Phi$ if, for every epistemic model $ E$, $E \models \Phi$ whenever $E \models \Psi$ for all $\Psi \in \Gamma$.

MEL is a logic that allows an agent to reason about another agent's beliefs or knowledge, not about one's own beliefs. If we admit an agent is aware of her knowledge, we can assume that when asked about the truth of a proposition $\varphi$ the agent can always say either that she believes it is true ($\Box \varphi$), false ($\Box \neg\varphi$) or does not know ($\Diamond \varphi \land \Diamond \neg \varphi$), which corresponds to a complete MEL base. The situation is different if one only wants to reason about what an agent knows of the beliefs of another agent. Then the former may partially ignore what the latter knows. A MEL base (any set of formulas in MEL) should be interpreted in this way. 

This point becomes clear when considering the formula $\Box \varphi \lor \Box \neg \varphi$, equivalent to $\neg(\Diamond \varphi \land \Diamond \neg \varphi)$, expressing the negation of ignorance. It says that the former agent knows that the latter agent knows the truth-value of $\varphi$, but the former agent does not know what it is.
 %**it seems it is correct as it is**}. Yes but we tried to be more explicit
 When neither $\Box \varphi$ nor $\Box \neg \varphi$ can be derived from a MEL base, such a formula makes no sense if the $\Box$ modality has an introspective flavor.\footnote{Namely, I cannot say ``I am not ignorant of the truth-value of $\varphi$, but I do not know which one." Besides, the usual modal introspection axioms 4: $\Box \varphi \to \Box\Box \varphi$ and 5: $\Diamond \varphi \to \Box\Diamond \varphi$ cannot be expressed in MEL.}

MEL can be axiomatized in a rather simple way, see  \citep{BD13}. The following {is} a possible set of axioms for MEL in the language of ${\cal L}_\Box$:

\begin{itemize}
\item[(CPL)] Axioms of CPL for ${\cal L}_\Box$-formulas
\item[(K)] $\Box(\varphi \to \psi) \to (\Box \varphi \to \Box \psi)$ 
\item[(D)] $\Box \varphi\to \Diamond \varphi$
\item[(Nec)] $\Box \varphi$, for each $\varphi \in {\cal L}$ that is a CPL tautology, i.e., if $Mod(\varphi) = \Omega$. 
\end{itemize}
The only inference rule is modus ponens. 
The corresponding notion of proof, denoted by $\vdash_{\rm MEL}$, is defined as usual from the above set of axioms and modus ponens. 

This set of axioms provides a sound and complete axiomatization of MEL, that is, it holds that, for any set of MEL formulas $\Gamma \cup \{\Phi\}$, $\Gamma \models \Phi$ iff $\Gamma \vdash_{\rm MEL} \Phi$. This is not surprising: MEL is just a standard propositional logic with additional axioms, whose propositional variables are the formulas of another propositional logic, and whose interpretations are subsets of interpretations of the latter. Namely, we have that $\Gamma \vdash_{\rm MEL} \Phi$ if and only if 
$\Gamma \cup Ax(\textrm{MEL}) \vdash_{\rm CPL}\Phi$, where $Ax(\textrm{MEL})$ is the set of all instances of the above axioms (CPL), (K), (D) and (Nec), and $\vdash_{\rm CPL}$ is the syntactic inference in standard propositional logic. All we have to do is to show that the additional axioms  ensure that a standard interpretation of the  ${\cal L}_\Box$ language corresponds indeed to a necessity measure based on a consistent epistemic state.

When the propositional language $\cal L$ is built over a finite set of propositional variables, for every interpretation $w \in \Omega$ there is a propositional formula whose only model is $w$. Indeed, if $V$
% = \{p_1, \ldots, p_n\}$ 
 is the finite set of propositional variables, then the maximal elementary conjunction (or min-term) $\sigma_w$, defined as
%$$\varphi_w := (mec's), i.e. $\varphi \equiv \lor_{\eta \in [\varphi]} \; \eta$, where each $\eta$ is  of the form
$$\sigma_w := (\bigwedge_{p \in V: w \models p} p ) \land (\bigwedge_{p\in V: w \not\models p} \neg p ), $$
is such that $Mod(\sigma_w) = \{w\}$. 
A similar result also holds for MEL. We have seen that models for MEL-formulas built from $V$ are subsets of $\Omega$, that is, $2^\Omega \setminus \{\emptyset\}$ is the set of MEL models. Then, for a given MEL-model $E \subseteq \Omega$, there is always a MEL-formula $\Sigma_E$ whose only model is $E$. Indeed, let $\varphi_E$ a propositional formula whose set of models is $E$ (for instance one can take $\varphi_E = \bigvee_{w \in E} \sigma_w$), and consider the MEL-formula
$$\Sigma_E := \Box \varphi_E \land \Delta \varphi_E,$$
where 
$$\Delta \varphi_E := \bigwedge_{w \in E}  \neg \Box(\varphi_E \land \neg \sigma_w)) .$$
Then one can check that $S \models_{MEL} \Sigma_E$ iff $S = E$. Namely: 
\begin{eqnarray*}
\{S \mid S \models \Box \varphi_E  \land (\bigwedge_{w \in E}  \neg \Box(\varphi_E \land \neg \sigma_w))\} &=& \\
\{S \mid S \subseteq E \} \cap  \bigcap_{w \in E} \{S \mid S \models  \neg \Box(\varphi_E \land \neg \sigma_w)\} &=& \\
\{S \mid S \subseteq E \} \cap  \bigcap_{w \in E} \{S \mid S \not\subseteq E \setminus \{w\} \} &=& \{E\}. 
\end{eqnarray*}
As a consequence, if $E$ and $E'$ are two different subsets of $\Omega$, then the formula $\Sigma_E \land \Sigma_{E'}$ has no models, or equivalently, by completeness of MEL,  the formula 
\begin{eqnarray}\label{incomp}
\neg(\Sigma_E \land \Sigma_{E'})
\end{eqnarray} 
is a theorem of MEL. Moreover,  the MEL-formula $\bigvee_{E \models \Phi} \Sigma_E$ has the same set of models as $\Phi$, hence, by completeness, MEL proves the equivalence
\begin{eqnarray} \label{equiv}
 \Phi \equiv \bigvee_{E \subseteq \Omega: E \models \Phi} \Sigma_E.
 \end{eqnarray}
 
\begin{remark} More generally, a MEL knowledge base or theory $\Gamma$ whose only model is $E$ is complete, that is, for each MEL-formula $\Psi$, either $\Psi$ or $\neg\Psi$ follows from $\Gamma$. As said earlier, intuitively, this base represents the knowledge of an agent 1 who knows (or believes (s)he knows) the epistemic state of another agent 2. That is, for every proposition $\varphi$ of the language $\mathcal L$, agent 1 knows whether for agent 2, $\varphi$ is true, false or unknown. When the MEL base of agent 1 is incomplete (i.e., when the MEL base has several models), there are propositions for which the latter cannot say whether for agent 2, such propositions are true, false or unknown. Note that for any proposition $\varphi$, an agent can always respond to the question: for you is $\varphi$ true, false or unknown? so an agent is aware of his epistemic state. But he may partially ignore the epistemic states of other agents, which is modeled by incomplete MEL knowledge bases.
\end{remark}

\section{Background II: {\L}ukasiewicz fuzzy logic } \label{Luk-logic}

{\L}ukasiewicz infinite-valued logic  \citep{L30} is one of the most prominent systems falling under the umbrella of Mathematical Fuzzy Logic, see the Handbooks  \citep{handbook1,handbook2,handbook3}. In fact, together with G\"odel  infinite-valued logic  \citep{Godel:ZumAussagen}, it was defined long before fuzzy logic as a discipline was born, and has received much attention since the fifties, when completeness results were proved in \citep{Rose-Rosser}, and via algebraic means by \citet{Chang:MVAlgebras,Chang2}. The latter developed the theory of MV-algebras, which is now widely studied in the literature.  For many details and results about {\L}ukasiewicz logic and MV-algebras the reader is referred to the monographs  \citep{CDM,MundiciAdvanced}. 
 
The language of {\L}ukasiewicz logic is built in the usual way
from a set of propositional variables,
one binary connective $\to_L$ (that is, {\L}ukasiewicz implication) and the truth constant $\bar{0}$, that we will also denote as $\bot$. A {\em valuation} $e$ maps every
propositional
variable to a real number from the unit interval $[0, 1]$ and extends to all
formulas
in the following way:\\

\begin{tabular}{r l l}
$e(\bar{0})$& = &0, \\
$e(\varphi \to_L \psi)$& = &$\min(1-e(\varphi)+e(\psi), 1)$.\\
\end{tabular}
\ \\ \\
Other interesting connectives can be defined from them. In Table \ref{connectives} one can find a list of the primitive and some derived connectives together with their definitions and associated truth-functions on the real unit interval $[0,1]$. 

\begin{table} 
\begin{center}
\begin{tabular}{l l l l}
Connective \mbox{} & Definition & Truth-function \\
\hline \vspace{-0.2cm}\\
Falsum: $\bar{0}$ & & 0, \\
Implication: $\to_L$ & & $\min(1, 1-x+y)$, \\
Truth: $\bar{1}$ & $\bar{1} := \varphi \to \varphi$ & 1,\\
Negation: $\neg_L$ & $\neg_L \varphi := \varphi \to_L \bar{0}$ & $1-x$,\\
Strong disjunction: $\oplus$ & $\varphi \oplus \psi := \neg \varphi \to_L \psi$ & $\min(1, x+y)$, \\
Strong conjunction: $\&$ & $\varphi \& \psi := \neg (\neg_L \varphi \oplus \neg_L \psi)$ & $\max(0, x+y-1)$,\\
Difference: $\ominus$ & $\varphi \ominus \psi := \varphi \;\&\; \neg_L \psi$ & $\max(0, x -y)$, \\
Equivalence: $\equiv$ & $\varphi \equiv \psi := (\varphi \to_L \psi) \& (\psi \to_L \varphi)$ & $1 - | x-y|$,\\
Weak conjunction: $\land$ & $\varphi \land \psi := \varphi \& (\varphi \to_L \psi)$ & $\min(x, y)$, \\
Weak disjunction: $\lor$ & $\varphi \lor \psi := \neg_L(\neg_L \varphi \land \neg_L \psi)$ & $\max(x, y)$
\end{tabular}
\end{center}
\caption{ Main connectives of {\L}ukasiewicz logic and their interpretations.}
 \label{connectives}
\end{table}

The truth-functions of {\L}ukasiewicz logic (primitive) connectives over the real unit interval $[0,1]$ define an algebra %with domain $[0,1]$ 
which is referred to as the {\em standard MV-algebra} and denoted by $[0, 1]_{MV}$. It generates, in universal algebraic terms, the equivalent algebraic semantics of {\L}ukasiewicz logic in the sense of Blok and Pigozzi, that is, the variety of MV-algebras  \citep{CDM,MundiciAdvanced}. 
A valuation $e$ is called a {\em model} of a set of formulas $T$ whenever $e(\varphi) = 1$ for each formula $\varphi \in T$. 

Axioms and rules of {\L}ukasiewicz Logic are the following, see e.g.\  \citep{CDM,H98}:  \\

\begin{tabular}{l l}

({\L1})& $\varphi \to_L (\psi \to_L \varphi)$\\

({\L}2)& $(\varphi \to_L \psi) \to_L ((\psi \to_L \chi) \to_L (\varphi \to_L \chi))$\\

({\L}3)& $(\neg \varphi \to_L \neg \psi) \to_L (\psi \to_L \varphi)$\\

({\L}4)& $((\varphi \to_L \psi) \to_L \psi) \to ((\psi \to_L \varphi) \to_L \varphi)$\\

(MP)& Modus ponens: from $\varphi$ and $\varphi \to_L \psi$ derive $\psi$\\ 
\end{tabular}
\ \\ 

From this axiomatic system, the notion of proof from a 
theory (a set
of formulas), denoted $\vdash_{\text \L}$, is defined as usual. 

The above axioms are tautologies: they are valid (i.e., they are evaluated to 1 by any valuation), and the rule of modus ponens preserves validity. Moreover, the following completeness result holds. 

\begin{theorem}\label{lukCompleteness}
The logic \L\ is complete for deductions from {\em finite} theories. That is, if $T$ is a finite theory, then $T \vdash_{\text \L} \varphi$ iff
$e(\varphi) = 1$ for each {\L}ukasiewicz valuation $e$ model of $T$.
\end{theorem}

\begin{remark}
This completeness result with respect to the standard semantics on $[0, 1]_{MV}$ is not valid for deductions from general (non-finite) theories. If one wants to enforce such  a stronger completeness result then one has to either to add to {\L}ukasiewicz logic the following infinitary rule of inference  \citep{Montagna}:
$$(IR) \; \frac{\varphi \to_L \psi^n, \mbox{ for each } n \in \mathbb{N}}{\neg \varphi \lor \psi}$$

\noindent where $\psi^n$ is a shorthand for $\psi \& \stackrel{n}{\ldots} \& \psi$, or to replace the standard real chain $[0, 1]_{MV}$ by an MV-chain on a hyperreal unit interval $[0, 1]^{*}$ as the domain of truth-values and hence allowing for infinitesimal and co-infinitesimal values  \citep{F08,apal}.  
\end{remark}

Observe that {\L}ukasiewicz logic does not satisfy the deduction theorem in full generality, it only satisfies the following {\em local} form: $T,\varphi\vdash_{\text \L} \psi$ iff there exists $ n \in \mathbb{N}$ such that $T\vdash_{\text \L} \varphi^n\to_L \psi$. There, $n$ depends on the formula $\varphi$.

\section{Probabilities on MEL formulas and belief functions}\label{P-on-MEL}

We first introduce the notion of probability function on MEL-formulas and then we will see that they are intimately related to belief functions on propositional formulas. 
\subsection{Probabilities of MEL formulas}

\begin{definition} A probability function on the language of MEL-formulas ${\cal  L}_{\Box}$ %built from a propositional language $\cal L$ 
is a mapping $\mu: {\cal  L}_{\Box} \to [0, 1]$ such that:
\begin{itemize}
\item[$(\Pi 1)$] $\mu(\neg \Phi) = 1 - \mu(\Phi)$

\item[$(\Pi 2)$] $\mu(\Phi \lor \Psi) = \mu(\Phi) + \mu(\Psi) - \mu(\Phi \land \Psi)$

\item[$(\Pi 3)$] $\mu(\Phi) = 1$, if $\vdash_{\rm MEL} \Phi$ 
\end{itemize}
\end{definition}

Note that, by $(\Pi 1)$ and $(\Pi 3)$ above, any probability $\mu$ on MEL-formulas satisfies that $\mu(\neg \Phi) = 0$ whenever $\vdash_{\rm MEL} \Phi$. Also, if $\Phi \to \Psi$ is a theorem of MEL, then $\mu(\neg \Phi \lor \Psi) = 1$, and by additivity $(\Pi 2)$, $1 = \mu(\neg\Phi) + \mu(\Psi) - \mu(\neg \Phi \land \Psi) \leq \mu(\neg\Phi) + \mu(\Psi) = 1-\mu(\Phi) + \mu(\Psi)$, and hence $\mu(\Phi) \leq \mu(\Psi)$. As a consequence, we have that probabilities on formulas respect logical equivalence. 

\begin{lemma} \label{preserv} Let $\mu: {\cal  L}_{MEL} \to [0, 1]$ be a probability on MEL-formulas. Then: 

(i) If $ \vdash_{\rm MEL}\Phi \to \Psi$, then $\mu(\Phi) \leq \mu(\Psi)$.

(ii)  If $\vdash_{\rm MEL} \Phi \equiv \Psi$, then $\mu(\Phi) = \mu(\Psi)$. 
\end{lemma}

Actually, probabilities on MEL-formulas are in one-to-one correspondence with probability distributions on MEL-models. {As the latter are non-empty sets, such probabilities define random sets on $\Omega$.}
%
%a belief function on formulas of $\cal L$ distribution $m_\mu$ on $2^\Omega$. Indeed, for every $S \subseteq \Omega$, $m_\mu(S) = \mu(\Box\varphi_S)$
%

Indeed, recall that $\Omega$ denotes the set of classical interpretations for the language $\cal L$, and let {$P: 2^{2^\Omega } \to [0, 1]$} be a probability measure on the power set of $\Omega$ {such that $P(\{\emptyset\}) =0$}, and define 
\begin{eqnarray} \label{eq1}
\mu_P(\Phi) = P(\{E \subseteq \Omega \mid E \models \Phi\}).
\end{eqnarray}
Then $\mu_P$ is a probability function on MEL-formulas. Notice that, subsets of $\Omega$ are indeed the elements (atoms) in $2^{2^\Omega}$, hence $\mu_P(\Phi) 
%= P(\{E \subseteq \Omega \mid E \models \Phi\}) 
= \sum_{E \models \Phi} P(\{E\})$. Then: 
\begin{itemize}
\item $\mu_P(\neg \Phi) = P(\{E \subseteq \Omega \mid E \not \models \Phi\}) = P(2^{2^\Omega} \setminus \{E \subseteq \Omega \mid E \models \Phi\}) = 1 - P(\{E \subseteq \Omega \mid E \models \Phi\}) = 1 - \mu_P(\Phi)$.

\item $\mu_P(\Phi \lor \Psi) = P(\{E \subseteq \Omega \mid E\models \Phi \lor \Psi\}) = P(\{E \subseteq \Omega \mid E\models \Phi\} \cup  \{E \subseteq \Omega \mid E\models \Psi\}) = P(\{E \subseteq \Omega \mid E\models \Phi\}) +  P\{E \subseteq \Omega \mid E\models \Psi\}) - P(\{E \subseteq \Omega \mid E\models \Phi\} \cap \{E \subseteq \Omega \mid E\models \Psi\}) =  \mu_P(\Phi) + \mu_P(\Psi) - \mu_P(\Phi \land \Psi)$.

\item if $\vdash_{\rm MEL} \Phi$, then $\mu_P(\Phi) = P(\{E \subseteq \Omega \mid E \models \Phi\}) =  P(\{E \subseteq \Omega\}) = P(2^{2^\Omega}) = 1$. 

\end{itemize}
Conversely, let  $\mu: {\cal  L}_{MEL} \to [0, 1]$ be a probability on MEL-formulas, and define the mapping  $P_\mu: 2^{\Omega} \to [0, 1]$ on $2^\Omega$ as follows: for every $E \subseteq \Omega$, 
\begin{eqnarray} \label{eq2}
 P_\mu(\{E\}) = \mu(\Sigma_E), 
 \end{eqnarray}
where $\Sigma_S$ is a MEL-formula having $S$ as the only model (such a formula always exists in a finite setting, recall Section \ref{sec-MEL}). So defined, $P_\mu$  is a probability distribution on $2^\Omega$, that is, $\sum_{E \subseteq \Omega}  P_\mu(\{E\}) = 1$. This follows from the fact that, if $\Phi$ is a theorem of MEL then, by $(\Pi 3)$, we have $\mu(\Phi) = 1$, and by \eqref{equiv} and $(\Pi 2)$, we have: 
$$1 = \mu(\Phi) = \mu(\bigvee_{E \subseteq \Omega} \Sigma_E) =  \sum_{E \subseteq \Omega} \mu(\Sigma_E) =  \sum_{E \subseteq \Omega} P_ {{\mu}}(\{E\}).$$
Then, by additivity, $P_\mu$ can be naturally extended to a probability measure on the whole space $2^{2^\Omega}$, that we will keep denoting it by $P_\mu$.  Namely, for each set of sets $\{S_1, \ldots, S_m\}$, we have $P_\mu(\{S_1, \ldots, S_m\}) = \sum_i P_\mu(\{S_i\})$. 
%It is easy to check that, so defined, $P_\mu$ is a probability on $2^{2^\Omega}$. 
Moreover, $\mu_{P_\mu} = \mu$ and $P_{\mu_P} = P$. Indeed, 
$$\mu_{P_\mu}(\Phi) = P_\mu(\{ E \mid E \models \Phi\}) = \sum_{E \models \Phi} P_\mu(E) \stackrel{\eqref{eq2}}{=}  \sum_{E \models \Phi}\mu(\Sigma_E) = \mu(\Phi), $$
the last equality being a consequence of \eqref{incomp} and \eqref{equiv}. On the other hand, 
\begin{align*}P_{\mu_P}(\{S_1, \ldots, S_m\})&=  \sum_i P_{\mu_P}(\{S_i\}) \stackrel{\eqref{eq2}}{=}  \sum_i \mu_P(\Sigma_{S_i}) \\ & \stackrel{\eqref{eq1}}{=}  \sum_i P(\{S_i\}) = P(\{S_1, \ldots, S_m\}). \end{align*}
 %P_\mu(\{ E \mid E \models \Phi\}) = \sum_{E \models \Phi} P_\mu(E) =  \sum_{E \models \Phi}\mu(\Sigma_E) = \mu(\Phi), $$

\subsection{Belief functions on propositional formulas}
 
{Here we show the connection between probability functions on MEL formulas and belief functions on propositional formulas.} 
%Next we introduce belief functions on propositional formulas. 

\begin{definition}
A belief function on formulas of $\cal L$ is a mapping $bel: {\cal L} \to [0, 1]$ such that:
\begin{itemize}
\item[(B1)] $bel(\varphi) = 1$ and $bel(\neg \varphi) = 0$, if $\vdash_{\rm CPL} \varphi$, 

\item[(B2)] $bel(\varphi_1 \lor ... \lor \varphi_n) \geq 
%\sum_{i} bel(\varphi)- \sum_{ij}bel(\varphi_i \land \varphi_j) + \sum_{ijk}bel(\varphi_i \land \varphi_j \land \varphi_k) - ...
\sum_{I \subseteq \{1, \dots,n\}}(-1)^{|I|+1}bel(\land_{i \in I}\varphi_i), \forall n \geq 2$, 
\hfill ($\infty$-monotonicity),\footnote{This property is characteristic of belief functions. If the inequality is replaced by equality, $bel$ is a probability measure, and it can be written at order 2 only. However, inequalities must be assumed at any order. See \cite{Shafer}.}

\item[(B3)] $bel(\varphi) = bel(\psi)$, whenever $\vdash_{\rm CPL} \varphi \equiv \psi$. 
\end{itemize}
\end{definition}

If the language is finitely generated, then it is easy to show that any such belief function on formulas of $\cal L$ is determined by a belief (set) function {on $\Omega$},  the set of classical interpretations of $\cal L$. Indeed, if $Bel: 2^\Omega \to [0, 1]$ is a belief function, then the corresponding mapping $bel$ on ${\cal L}$ defined as 
\begin{eqnarray} \label{bel}
bel(\varphi) = Bel(\{w \in \Omega \mid w\models \varphi \})
\end{eqnarray}
is clearly a belief function on formulas of $\cal L$. Conversely, if $bel$ is a belief function on formulas of $\cal L$, then we can define a mapping $Bel: 2^\Omega \to [0, 1]$ by putting, for every $E \subseteq \Omega$, 
\begin{eqnarray} \label{Bel}
Bel(E) = bel(\varphi_E),
\end{eqnarray}
where {$\varphi_E$ is a propositional formula whose set of models} 
%such that the only epistemic model of $\Sigma_E$ 
is  $E$. 
%As we have seen, 
This formula always exists in the finite setting, and moreover this is well defined because belief functions on formulas respect classical equivalence. \\

\noindent \underline{Notation convention}: From now on, without danger of confusion, if $Bel$ is a belief function on subsets of interpretations we will denote its corresponding belief function on formulas by $bel$, and conversely.  \\

Finally, let us explicitly show the one-to-one relationship between probabilities on MEL-formulas and belief functions on propositional formulas.

\begin{proposition} \label{prop4.4} A mapping  $bel: {\cal L} \to [0, 1]$ is a belief function on formulas from ${\cal L}$ iff there is a probability $\mu$ on MEL-formulas from ${\cal L}_{\Box}$  such that $bel(\varphi) = \mu(\Box \varphi)$.  
\end{proposition}

\begin{proof}  Let $bel: {\cal L} \to [0, 1]$ be a belief function on formulas of $\cal L$. Then let $Bel: 2^\Omega \to [0, 1]$ be the belief function on $\Omega$  
%subsets of interpretations 
as defined in \eqref{Bel}. Then let $m_{bel}: 2^\Omega \to [0, 1]$ be the corresponding mass distribution on $2^\Omega$, obtained by the M\"obius transform of $Bel$, that is, the unique set function such that, for every $E \subseteq \Omega$, $Bel(E) = \sum_{F \subseteq E} m_{bel}(F)$. Let us denote by $P_{bel}$ the corresponding probability measure on $2^{2^\Omega}$, and let  $\mu_{P_{bel}}$ be its associated probability function on formulas according to \eqref{eq1}.  Then we have:
$$bel(\varphi)= Bel(Mod(\varphi)) = \sum_{F \subseteq Mod(\varphi)} m_{bel}(F) = P_{bel}(\{F \mid F \models \Box \varphi\}) = \mu_{P_{bel}}(\Box \varphi).$$
Conversely, let us consider a  probability $\mu: {\cal L}_{\Box} \to [0,1]$ on MEL-formulas, and let us check that the mapping $bel_\mu: {\cal L} \to [0, 1]$ defined as $$bel_\mu(\varphi) = \mu(\Box \varphi)$$
is a belief function on propositional formulas. 
Indeed, let us check that $bel_\mu$ satisfies the conditions  (B1), (B2) and (B3). As for (B1), if $\vdash_{\rm CPL} \varphi$, then $\vdash_{\rm MEL} \Box\varphi$ as well, and hence, by $(\Pi  3)$, we have $bel_\mu(\varphi) = \mu(\Box \varphi) = 1$. Also, $\Box \neg \varphi \to \Diamond \neg \varphi$ is an instantiation of Axiom (D), but $\Diamond \neg \varphi$ is equivalent to $\neg \Box \varphi$, and hence $\Box \neg \varphi \to \neg \Box \varphi$ is a MEL theorem. By (i) of Lemma \ref{preserv}, we have $\mu(\Box \neg \varphi) \leq \mu(\neg \Box \varphi)$, but  $\mu(\neg \Box \varphi) = 1 - \mu(\Box \varphi) = 1 - 1 = 0$. 

As for (B2), note that for any propositions $\varphi_1, \ldots, \varphi_n$, the MEL-formula $$(\bigvee_{i=1,n} \Box \varphi_i) \to \Box(\bigvee_{i=1,n} \varphi_i)$$ is a theorem of MEL, and hence, by Lemma \ref{preserv}, $\mu(\bigvee_{i=1,n} \Box \varphi_i) \leq \mu(\Box(\bigvee_{i=1,n} \varphi_i)) )= bel_\mu(\bigvee_{i=1,n} \varphi_i)$.

On the other hand, by applying  iteratively the additivity property of $\mu$ we get: 
$$
\begin{array}{l}
\mu(\bigvee_{i=1,n} \Box \varphi_i) =\\
= \sum_{i} \mu(\Box\varphi)- \sum_{ij} \mu(\Box(\varphi_i \land \varphi_j)) + \sum_{ijk} \mu(\Box(\varphi_i \land \varphi_j \land \varphi_k)) - ...  = \\
 =  \sum_{i} bel_\mu(\varphi)- \sum_{ij}bel_\mu(\varphi_i \land \varphi_j) + \sum_{ijk}bel_\mu(\varphi_i \land \varphi_j \land \varphi_k) - ... \\
\end{array}
$$
Therefore, we get:
$$
\begin{array}{l}
bel_\mu(\bigvee_{i=1,n} \varphi_i) = \mu(\bigvee_{i=1,n} \Box \varphi_i)  \geq  \\
\geq \sum_{i}bel_\mu(\varphi)- \sum_{ij}bel_\mu(\varphi_i \land \varphi_j) + \sum_{ijk}bel_\mu(\varphi_i \land \varphi_j \land \varphi_k) - ... \\
\end{array}
$$
that is, $bel_\mu$ also satisfies (B2). Finally (B3) is a direct consequence of (ii) of Lemma \ref{preserv}. 
\end{proof}
The mass assignment $m_{Bel}(E)$ is equal to the probability $P(\Sigma_E)$, where the modal formula $\Sigma_E$ introduced earlier has $E$ for only model (it expresses the idea that all that is known is $E$). It has been recalled that $m_{Bel}$ is the M\"obius transform of $Bel$ 
$$m_{Bel}(E)= \sum_{A: E\subseteq A} (-1)^{|A\setminus E|} Bel(A).$$
 In  \citep{BD13}, it has been shown that this M\"obius transform formula, in the all-or-nothing case ($m_{Bel}(E)= 1$), reduces to the modal formula $\Sigma_E$, which is in agreement with the belief function logic presented in the next section.

\section{A two-layer modal logic over {\L}ukasiewicz logic for belief functions}\label{BF-log}

As should be clear from the previous section, the problem of defining a logic for belief functions on classical propositions can be reduced to defining a probability logic over MEL. This is what we do in this section, using the approach by \citep{HGE95} and \citep{H98} to define probability logics as modal theories over {\L}ukasiewicz logic, and greatly simplifying the approach in  \citep{GHE01,GHE03} that defined a logic for belief functions as a probabilistic logic over the whole modal logic S5. 

\subsection{The logic FB(MEL, \!\L): syntax, semantics and proof system }
 We consider now the two-layer  logic FB(MEL, \!\L), where $\L$ is {\L}ukasiewicz logic, to reason about the probability of MEL formulas. The guiding idea, as already mentioned in the Introduction and following from  \citep{H98}, is to formalize the fact that belief functions on classical propositions can be seen as probabilities on modal formulas. To this end, we proceed like \citet{H98} who formalises probability as a fuzzy modality $P$ in the setting of {\L}ukasiewicz logic
 to reason about the probability of classical propositions;
but this time, in the following, the modality $P$ will apply not to classical propositions but to MEL formulas.

The language of FB(MEL, \!\L) consists of the following two kinds of formulas:  

\begin{itemize}
\item[-]  {\em MEL-formulas}: taken from the language ${\cal L}_\Box$, built over a finitely generated propositional language $\cal L$, as defined in Section \ref{sec-MEL}. General MEL-formulas will be denoted by capital greek letters $\Phi, \Psi, ...$

\item[-]  {\em P-formulas}: atomic P-formulas are of the form $P \Phi$, where $\Phi$ is a MEL formula; compound P-formulas are propositional combinations of atomic ones with {\L}ukasiewicz connectives $\to_L, \&, \neg_L$, and will be denoted by capital letters $\mathcal{A, B}, \ldots$. In particular, we will call 
{\em B-formulas} to those built from P-formulas of the form $P \Box \varphi$, that will be also denoted as $B\varphi$.
\end{itemize}

The semantics of FB(MEL, \!\L) is basically given by belief functions $Bel: 2^\Omega \to [0, 1]$ on subsets of the set $\Omega$ of interpretations for the propositional language $\cal L$, or equivalently by their corresponding belief functions on formulas, i.e., $bel:{\cal L} \to [0, 1]$.
%In the following, given a propositional formula $\varphi \in {\cal L}$, we will denote by $[\varphi]$ its sets of models, i.e.\ $[\varphi] = \{w \in \Omega \mid w \models \varphi\}$. 
Recall that epistemic MEL-models are just non-empty subsets of interpretations from $\Omega$, hence a belief function $Bel: 2^\Omega \to [0, 1]$ is actually a belief function on the set of MEL-models.

Then, the evaluation of an arbitrary atomic P-formula $\Phi$ by a belief function  on formulas $bel$ 
%$K = (\Omega, Bel)$ 
is {defined} as follows (recall the definition of the probability $P_{bel}$ on $2^{2^\Omega}$ introduced in the proof of Proposition \ref{prop4.4}): 
%
%$$\| P \Phi \|_K = P_{bel}(\Phi) = \mu_{bel}(\{ E \subseteq \Omega \mid E \models \Phi \}) = \sum_{E \models \Phi} \mu_{bel}(E).$$
$$\| P \Phi \|_{bel} = P_{bel}(\{ E \subseteq \Omega \mid E \models \Phi \}) = \sum_{E \models \Phi} P_{bel}(E).$$
Note that, in the particular case when $\Phi$ is a B-formula, that is, a formula of the form $\Phi = \Box \varphi$, then 
$$\| P \Box \varphi \|_{bel} =   \sum_{E \models \Phi} P_{bel}(E) = \sum_{E \subseteq Mod(\varphi)} P_{bel}(E) = Bel(Mod(\varphi)) = bel(\varphi), $$
and hence,  B-formulas are faithfully interpreted by belief functions. 

The truth-evaluation of compound P-formulas is then defined by using {\L}ukasiewicz truth-functions. For instance, 
\begin{eqnarray*}
\| P \Phi  \to_L P \Psi \|_{bel} &=& \min(1, 1- \| P \Phi \|_{bel} + \| P \Psi \|_{bel})\\
\| P \Phi  \;\&\; P \Psi \|_{bel} &=& \max(0,  \| P \Phi \|_{bel} + \| P \Psi \|_{bel} -1) \\
\| P \Phi  \land P \Psi \|_{bel} &=& \min(\| P \Phi \|_{bel}, \| P \Psi \|_{bel}) \\
\| P \Phi  \lor P \Psi \|_{bel} &=& \max(\| P \Phi \|_{bel}, \| P \Psi \|_{bel}) \\
\| \neg_L P \Phi \|_{bel} &=& 1 -\| P \Phi \|_{bel} 
\end{eqnarray*}
The above semantics allows us to define a natural notion of logical consequence for P-formulas.

\begin{definition}
For any set  $T \cup \{\mathcal A\}$ of P-formulas, we say that $\mathcal A$ logically follows from $T$ in FB(MEL, \!\L), written $T \models_{FB} \mathcal A$, if for any belief function on formulas $bel$, we have that $\| \mathcal B \|_{bel} = 1$ for all $\mathcal B \in T$ implies $\| \mathcal A \|_{bel} = 1$ as well. 
 \end{definition}

%Similarly to the case of the probability logic FP(\L),  we consider for the  belief function logic FB(MEL, \L) the following axioms and rules: 
Now we introduce the following axiomatic system for the belief function logic FB(MEL, \!\L): 

\begin{itemize}
\item Axioms and rules of MEL for MEL-formulas
\item Axioms and rules of {\L}ukasiewicz logic for P-formulas
\item Probabilistic axioms (for P-formulas):

%(A1) $\neg P \bot$
%
%(A2) $P(\Phi \to \Psi) \to_L (P\Phi \to_L P\Psi)$
%
%(A3) $P(\Phi \lor \Psi) \equiv P(\Phi) \oplus (P\Psi  \ominus P(\Phi \land \Psi))$
%
%\item Necessitation rule: if $\Phi$ is a MEL theorem, then infer $P\Phi$
%
%

\begin{tabular}{l l}
(FP0) \ \ \ \ & $P\Phi$, for $\Phi$ being a theorem of MEL \\
(FP1) \ \ \ \ & $P(\Phi \to \Psi) \to_L (P\Phi \to_L P \Psi)$\\
(FP2)& $P(\neg \Phi) \equiv \neg_L P\Phi$\\
(FP3)& $P(\Phi \vee \Psi) \equiv (P\Phi \to_L P(\Phi \wedge \Psi)) \to_L P\Psi $ \\
%(FP4)& $P(\varphi \wedge \psi \mid \chi) \equiv P(\psi \mid \varphi \wedge \chi) \odot P(\varphi \mid \chi)$\\
%(FP5)& $P(\chi \mid \chi)$\\
\end{tabular}
\end{itemize}
{(FP3) might look a bit mysterious, but for a probability measure $P$ we indeed have that 
$\min(1, 1 - \min(1, 1 - P(A) + P(A \cap B)) + P(B))= \min(1, \max(0,  P(A) - P(A \cap B)) + P(B))= \min(1, \max(P(B), P(A)  + P(B)- P(A \cap B))=P(A \cup B)$.}
{Actually, this axiom can be shown to be logically equivalent to this other scheme
$$ P(\Phi \vee \Psi) \equiv P\Phi \oplus (P\Psi \ominus P(\Phi \land \Psi))$$
that better expresses the finite additivity property of the modal operator $P$. }

Using these axioms and rules, one can then define a syntactic notion of proof for  P-formulas. 

\begin{definition}
 Let $T \cup \{A\}$ be a set of P-formulas. Then $A$ syntactically follows from $T$ in FB(MEL, \L), written $T \vdash_{FB} A$, if it can be proved from $T$ in the natural way from the above axioms and inference rules. 
 \end{definition}

Note that, since $(\Box \varphi \lor \Box \psi) \to \Box (\varphi \lor \psi)$ is a theorem of MEL, then by the necessitation rule, $\vdash_{FB} P((\Box \varphi \lor \Box \psi) \to \Box (\varphi \lor \psi))$, and by  (FP1), we have that $\vdash_{F{B}} P(\Box \varphi \lor \Box \psi) \to_L P\Box (\varphi \lor \psi)$. In general, for any $n$, the formula $$P(\bigvee_{i=1,n} \Box \varphi_i) \to_L P( \Box\bigvee_{i=1,n} \varphi_i),$$  
%$\vdash_{FP} P(\Box \varphi \lor \Box \psi) \to_L P\Box (\varphi \lor \psi)$. 
%In general, for any $n$, the formula $$P(\bigvee_{i=1,n} \Box \varphi_i) \to_L P( \Box\bigvee_{i=1,n} \varphi_i)$$ 
is a theorem of FB(MEL, \L). 
\subsection{Soundness and completeness}

Our next task is to show that the above axiom system is sound and complete for the belief function semantics introduced above. To provide a proof we first need some preparation. The idea is to translate proofs from finite theories in  FB(MEL, \L) into proofs from larger but still finite theories in {\L}ukasiewicz logic, and then to take advantage of completeness of {\L} to show that models of this larger theory correspond to belief functions on formulas. 
% Completeness:  for any set  $T \cup \{A\}$ of P-formulas, it  holds that $T \models_{FB} A$ iff $T \vdash_{FPB} A$. 

The usual strategy to prove completeness of probabilistic modal logics like \FB\ w.r.t. probabilistic models consists in the following steps (see e.g.\  \citep{FGM11,CN14} for more details):
\begin{enumerate}

\item[(S1)]  First of all we define a syntactic translation $^\circ$ from modal to propositional formulas  of \L ukasiewicz logic by interpreting every atomic modal formula $P\Phi$ in a new propositional variable $p_\Phi$ and extending $^\circ$ to compound modal formulas by commuting with connectives: 

- $(P\Phi)^\circ = p_\Phi$

- $(\neg_L \mathcal A)^\circ = \neg_L\mathcal A^\circ$

- $(\mathcal A *\mathcal B)^\circ =\mathcal A^\circ *\mathcal B^\circ$, for $* \in \{\to_L, \&, \lor, \land\}$

\item[(S2)]  The translation of all instances of the axioms  (FP0)-(FP3), 
%together with the set $\{p_\varphi \mid\; \vdash\varphi\}$ which encodes the propositional translation of (FP0), 
gives rise to a propositional \L-theory ${\bf FP}^\circ$ such that, for every (finite) set of modal formulas $T\cup\{\mathcal A\}$, 
\begin{center}
$T\vdash_{FB} \mathcal A$ iff $T^\circ \cup {\bf FP}^\circ\vdash_{\text{\L}} \mathcal A^\circ$, 
\end{center}
see for instance  \citep{FG07,FGM11} and  \citep{H98}.
\end{enumerate}

Now, assume that  $T\not\vdash_{FP} \mathcal A$ and hence $T^\circ \cup {\bf FP}^\circ\not\vdash_{\text{\L}} \mathcal A^\circ$. Now, even if $T$ is finite, ${\bf FP}^\circ$ is an infinite theory, and since {\L}ukasiewicz logic is not strongly standard complete, i.e.\ standard complete with respect to deductions from infinite theories, then we cannot guarantee in principle the existence of a {\L}ukasiewicz valuation in the real unit interval $[0, 1]$ that is a model of $T^\circ \cup {\bf FP}^\circ$ and a countermodel of  $A^\circ$. However, since we are assuming that the propositional language $\cal L$ is finitely generated, both $\cal L$ and ${\cal L}_{\Box}$ contain only finitely-many different formulas modulo logical equivalence. Thus, for each propositional formula $\varphi$ we can choose one representative $\varphi^*$ of its equivalence class. Moreover, the MEL sublanguage of ${\cal L}_{\Box}$ generated by atomic modal formulas of the form $\Box \varphi^*$, although infinite, has also finitely-many non-logically equivalent formulas, and so again, for any MEL-formula $\Phi$ in this sublanguage we can pick a representative $\Phi^*$ of its equivalence class. Let us denote by  ${\cal L}^*_{MEL}$  this finite set of MEL-fomulas. Finally let $({\bf FP}^\circ)^*$ be the finite subtheory of $ {\bf FP}^\circ$ built from instances of the axioms (FP0)-(FP3) with formulas of  ${\cal L}^*_{\Box}$. Then one can show the following chain of equivalences:
\begin{center}
 $T\vdash_{FB} \mathcal A$ \; iff \; $T^\circ \cup {\bf FP}^\circ\vdash_{\text{\L}} \mathcal A^\circ$ \; iff  \; $T^\circ \cup({\bf FP}^\circ)^*\vdash_{\text{\L}} \mathcal A^\circ$.
\end{center}
Now,  $T^\circ \cup({\bf FP}^\circ)^*$ is finite, and thus, by completeness of {\L}ukasiewicz logic, there is an $[0, 1]$-valued {\L}ukasiewicz valuation $e$ such that $e(T^\circ \cup({\bf FP}^\circ)^*) = 1$ and $e({\cal A}^\circ) < 1$. But since $e$ is a model of the (translations of the) axioms (FP0)-(FP3), then the mapping $\mu_e$ on MEL-formulas $\Phi$ defined as 
$$\mu_e(\Phi) = e(p_{\Phi^\circ})$$
is a probability on MEL-formulas, and hence the mapping $bel_{\mu_e}: {\cal L} \to [0, 1]$ defined as 
$$bel_{\mu_e}(\varphi) = \mu_e(P\varphi)$$
is a belief function on formulas that is a model of $T$ but does not satisfy $\cal A$. 

Thus, the following completeness result holds. 

\begin{theorem} \label{theo2} {\bf (Belief function completeness of  \FB)} 
Let $T$ be a finite modal theory over \FB\ and $\mathcal  A$ a 
P-formula of \FB. Then $T\vdash_{FB} \mathcal A$ iff $\| \mathcal A \|_{bel} = 1$ for each
 belief function $bel$  model of $T$.
\end{theorem}
\subsection{Reasoning with comparative beliefs}

 \FB\ can be used to reason in a purely qualitative way about comparative statements about belief functions on propositions, as done in  \citep{Harmanec1}, by exploiting the fact a \FB-formula of the form $B\psi \to B\varphi$ is 1-true in a model defined by a belief function $bel$ iff $bel(\psi) \leq bel(\varphi)$. 
 Therefore, if we represent the statement ``the event $\varphi$ is at least as believed as the event $\psi$'' as $\psi \triangleleft \varphi$, then an inference of the form 
$$\mbox{from } \psi_1 \triangleleft \varphi_1, \ldots, \psi_n \triangleleft \varphi_n  \mbox{ infer }  \chi \triangleleft \nu$$
can be faithfully captured by the following derivation in \FB
$$B\psi_1 \to B\varphi_1, \ldots, B\psi_n \to B\varphi_n \vdash_{FB} B\chi \to B\nu .$$

{
In fact,  \citet{Wong91} axiomatically characterise the strict comparative relations $\succ$ among subsets of a finite set $X$ induced by belief functions on $2^X$. Here we consider an equivalent set of axioms for the non-strict comparative relation $\succeq$. Indeed, let us consider the following postulates for a relation $\succeq$ on $2^X$: 
$$
\begin{array}{ll}
(BW1) & (2^X, \succeq)  \mbox{ is a total pre-order, that is for any $A, B, C \subseteq X$,} \\ 
& - A\succeq A \\
& - A \succeq B, B \succeq C  \mbox{ implies } A \succeq C \\
& - A \succeq B \mbox{ or } B \succeq A \\
%(B1) & A \succ B  \mbox{ implies } B \not\succ A \\
%(BW1) & A \succ B  \mbox{ implies } A  \succeq B \\ \\
%
% (B2) &  A \not\succ B, B \not\succ C   \mbox{ implies }  A \not\succ C\\
% & A \succ C    \mbox{ implies }  A \succ B \mbox{ or } B \succ C \\
% (BW2) & B \succeq A, C \succeq B    \mbox{ implies } C \succeq A \\ \\
% 
%(B3) & A \supseteq B    \mbox{ implies } B \not\succ A \\
(BW2) & A \supseteq B    \mbox{ implies } A \succeq B \\ 

%(B4) &  (A \supset B, A \cap C = \emptyset)  \mbox{ implies } (\mbox{if } A \succ B \mbox{ then }  A \cup C \succ B \cup C)\\
(BW3) &   \mbox{Whenever }A \supseteq B  \mbox{ and } A \cap C = \emptyset, \\ &\mbox{ we have that } (\mbox{if }  B \cup C  \succeq A \cup C  \mbox{ then }  B \succeq A) \\

%(B5) &  X \succ \emptyset\\is acuall
(BW4) &  \emptyset \not \succeq X \\
\end{array}
 $$
 Note that BW2 expresses monotonicity of belief functions with respect to inclusion; so due to the  condition $A \supseteq B$ in BW3,  the assumption $B \cup C  \succeq A \cup C $  is actually equivalent to $B \cup C  \sim A \cup C $, where $\sim$ is the equivalence relation contained in $\succeq$ (and likewise for the conclusion $B \succeq A$). So BW3 is a condition of equivalence preservation.
 BW3 is in fact a weakening of  the axiom of comparative probability  \citep{DeF23} (where the condition $A \supseteq B$ is not required, but $B \cap C = \emptyset$ as well, and the last implication is an equivalence).
 
  Then the following representation is an equivalent reformulation of \cite[Th. 4]{Wong91}. 
 
 \begin{theorem} Let $X$ be a finite set and $\succeq$ a  relation on $2^X$. Then, there exists a belief function $Bel: 2^X \to [0, 1]$ such that 
$$ A \succeq B   \mbox{ iff }   Bel(A) \geq Bel(B)$$
if and only if the relation $\succeq$ satisfies the properties (BW1)-(BW4). 
\end{theorem}

Now, in FB(MEL, \L) let us introduce the notation $\varphi \triangleright_B \psi$ to refer to the  formula $B\varphi  \to_L B\psi$, that is, we define:  
%by  the following notation convention: 
%we can consider the following definable binary connective $\triangleright_B$: 
$$ \varphi \triangleright_B \psi := B\varphi  \to_L B\psi.$$
Looking at $\triangleright_B$ as a sort of propositional binary connective,  it turns out that $\triangleright_B$ precisely captures the semantics of the belief comparative relation $\succeq$ in the sense that  $\varphi \triangleright_B \psi$ is 1-true in a belief function model if, and only if,  the belief degree of $\varphi$ is greater than or equal to the belief degree of $\psi$. 

Then, by completeness of FB(MEL, \L), the following analogues of the properties (BW1)-(BW4) hold: 

\begin{itemize}

\item[(BW1)] $\vdash_{FB} \varphi \triangleright_B \varphi$
\\
 $\vdash_{FB} (\varphi \triangleright_B \psi) \&  (\psi \triangleright_B \chi) \to_L  (\varphi \triangleright_B \chi)$
\\
$\vdash_{FB} (\varphi \triangleright_B \psi) \lor  (\psi \triangleright_B \varphi)$

\item[(BW2)] If $\vdash_{CPL} \varphi \to \psi$ then $\vdash_{FB}  \varphi \triangleright_B \psi$ 

\item[(BW3)] If $\vdash_{CPL} (\varphi \to \psi)  \land \neg(\psi \land \chi)$ then $ (\psi \lor \chi) \triangleright_B (\varphi \lor \chi) \vdash_{FB} (\psi \triangleright_B \varphi) 
$

\item[(BW4)] $\vdash_{FB}   \neg_L(\top \triangleright_B \bot)$

\end{itemize}
where we recall that $\vdash_{CPL}$ denotes proof in classical propositional logic.

Nonetheless, taking advantage of the additive flavour of the semantics of {\L}ukasiewicz connectives, \FB\ also allows reasoning about  statements with a more quantitative flavour, e.g. an statement like ``$\varphi$ is as twice as believed as  $\psi$'' can be represented by the formula
$$B\psi \oplus B\psi \to_L B\varphi.$$
Indeed, recalling from Section \ref{Luk-logic} the interpretation of {\L}ukasiewicz strong disjunction $\oplus$, for any belief function on formulas $bel$, $\| B\psi \oplus B\psi \to_L B\varphi \|_{bel} = 1$ iff $bel(\varphi) \geq \min(2 \cdot bel(\psi), 1)$. 

%\red{Write axiom B4 of Wong in the logic: if $\phi, \chi, \psi$ are pairwise contradictory propositions, then
%$(B\phi\to B\psi) \to (B(\phi\lor\chi) \to B(\psi\lor\chi))$
%should be a theorem (or maybe $\Delta((B\phi\to B\psi)$
% $\to ( (B(\phi\lor\chi)\to B(\psi\lor\chi)))$)} 
 %\to B(\psi\lor\chi))}
%

\section{Reasoning quantitatively by adding truth-constants} \label{truthcons}

We have seen in the previous section that \FB\ is a suitable formalism to reason about belief functions basically in a qualitative or comparative way. If one wants to explicitly reason about numerical statements, like ``the belief of $\varphi$ is (or is at least, at most) 0.6'',  an elegant solution is to replace in \FB\  the outer logic \L\ by the so-called {\em Rational Pavelka logic}, denoted RPL,  that is the expansion of \L\ with rational truth-constants.

Before introducing the logic \FBR, 
we first briefly recall the main notions and properties of RPL. 
Following \citet{H98}, the language of RPL is the language of {\L}ukasiewicz logic expanded with countably-many truth-constants $\overline r$, one for each rational $r \in [0, 1]$. 
The evaluation of RPL formulas is as in {\L}ukasiewicz logic, with the proviso that valuations evaluate truth-constants to their intended value, that is, for any rational $r \in [0, 1]$ and any valuation $e$,  $e( \overline r) = r$. 
Note that, for any valuation $e$,  $e(\overline{r} \to \varphi) = 1$ iff $e(\varphi) \geq r$, and  $e(\overline{r} \equiv \varphi) = 1$ iff $e(\varphi) = r$. 

Axioms and rules of RPL are those of \L\ plus the following countable set of book-keeping axioms for truth-constants: \\

\begin{tabular}{l l}

(BK$_\to$)& $\overline{r} \to_L \overline{s} \equiv \overline{\min(1, 1-r+s)}$, for any rationals $r, s \in [0, 1].$\\

\end{tabular}
\ \\ \\
{Since all the other {\L}ukasiewicz connectives are definable from $\to_L$ and the constant $\overline 0$, similar book-keeping axioms are derivable, for instance, \\

\begin{tabular}{l l}

(BK$_{\&}$) & $\overline{r} \& \overline{s} \equiv \overline{\max(r+s-1, 0)}$, for any rationals $r, s \in [0, 1],$\\
(BK$_{\neg}$)& $\neg_L \overline{s} \equiv \overline{1-r}$, for any rational $r \in [0, 1].$\\

\end{tabular} 
}
\ \\ \\

The notion of proof is defined as in  {\L}ukasiewicz logic, and the deducibility relation will be denoted by $\vdash_{RPL}$. Moreover, completeness of {\L}ukasiewicz logic smoothly extends to RPL as follows: if $T$ is finite theory over RPL, then $T \vdash_{\text{RPL}} \varphi$ iff
$e(\varphi) = 1$ for any RPL-valuation $e$ model of $T$.

It is customary in RPL to introduce the following notions: 
%also enjoys a sort of infinitary completeness result, known as {\em Pavelka-style completeness}; 
for any set of RPL formulas $T \cup \{\varphi\}$, define: \\

\noindent - the {\em truth degree} of $\varphi$ in $T$ as:  \\ \centerline{$\| \varphi \|_T =  \inf \{e(\varphi) : e \mbox{ RPL-valuation model of } T \}$,} \\
%\sup_{T^f \subseteq T} \inf_e \{e(\bigwedge T^f) \to e(\varphi) \}$

\noindent - the {\em provability degree} of $\varphi$ from $T$ as: \\ \centerline{$ \mid \varphi \mid _T\; = \sup\{ r \in [0, 1]_{\mathbb Q} \mid T \vdash_{RPL} \overline{r} \to \varphi \}$.} \\

Then, the so-called {\em Pavelka-style completeness} for RPL refers to the result that  $\mid \varphi \mid _T\; = \| \varphi \|_T $ holds for any arbitrary (non necessarily finite) theory $T$  \citep{H98}. Note that $ \mid \varphi \mid _T\ = \alpha$ does not guarantee that $T \vdash_{RPL} \overline{\alpha} \to \varphi$, even if $\alpha$ is rational. However, 
if $T$ is finite, we can restrict ourselves to rational-valued {\L}ukasiewicz valuations and get the following result, 
proved in  \citep{H98}. 

\begin{proposition} \label{rational-comp} If $T$ is a finite theory over RPL, then: % the following conditions hold: 
\begin{itemize}
\item  $\| \varphi \| _T$ is rational, hence if $\| \varphi \| _T =r$ then $T \vdash_{RPL} \overline{r} \to \varphi$.
\end{itemize}
In particular,  $\| \varphi \| _T =1$ iff $T \vdash_{RPL} \varphi$. 
\end{proposition}

Finally, let us introduce the logic \FBR. To this end, we only have to expand the language of P-formulas with truth-constants $\overline r$, one for every rational $r \in [0, 1]$, and in the axiomatic definition of \FB\ we add the book-keeping axioms (BK) of RPL. The semantics remains basically the same as for \FB, given by belief functions on formulas $bel$, with the obvious further requirement that $\| \overline{r} \|_{bel} = r$ for each rational $r$ when evaluating compound P-formulas involving truth-constants. We will denote the notion of proof in \FBR\ by $\vdash_{FBR}$.

\begin{theorem} \label{theo3} {\bf (Belief function completeness of  \FBR)} 
Let $T$ be a finite modal theory over \FBR\ and $\cal A$ a P-formula of \FBR. Then $T\vdash_{FBR} A$ iff $\| {\cal A} \|_{bel} = 1$ for each
 belief function $bel$  model of $T$.
\end{theorem}

To  conclude this section, we can show that a sort of a graded modus ponens rule, valid in probabilistic logic and similar to the one of possibilistic logic, is also valid in FB(MEL, RPL) for formulas of the form $\overline r\to B\phi$ (encoding inequalities of the kind $bel(\phi) \geq r$), getting a sort of belief function counterpart of standard possibilistic logic  \citep{DLP94}. 

\begin{proposition} The following deduction holds: 
$$\{\overline r\to_L B\varphi, \overline s\to_L B(\varphi \to \psi)\}\vdash_{FBR}\overline{\max(r + s -1, 0)} \to_L B\psi .$$
\end{proposition}

\begin{proof}
Since $\Box(\varphi \to \psi)  \to (\Box \varphi \to \Box \psi)$ is an axiom of MEL, by axiom (FP0), 
$$P(\Box(\varphi \to \psi)  \to (\Box \varphi \to \Box \psi)),$$
is a theorem of FB(MEL, RPL), and by axiom (FP1), FB(MEL, RPL) also proves: 
$$P\Box(\varphi \to \psi)  \to_L P(\Box \varphi \to \Box \psi) .$$
By using again axiom (FP1) on the right-hand side of the above implication, we can prove:
$$P\Box(\varphi \to \psi)  \to_L (P\Box \varphi \to_L P\Box \psi), 
$$
 that is in fact an analog of axiom (FP0) for B = P$\Box$, namely
 $$B(\varphi \to \psi)  \to_L (B \varphi \to_L B \psi).$$
Now, using the residuation law of \L,\footnote{Namely, the fact that $\chi\to_L(\varphi \to_L\psi)$ is equivalent to 
$\chi\&\varphi \to_L\psi$ in {\L}ukasiewicz logic.}  
we have that FB(MEL, RPL) proves the following theorem:
$$(B \varphi \& B(\varphi \to \psi) ) \to_L B \psi . $$
Hence, using this theorem, we can finally prove that the following chain of deductions hold in FB(MEL, RPL): 
$$\begin{array}{lll}
\overline{r} \to_L B \varphi, \overline{s} \to_L B(\varphi \to \psi) &\vdash_{FB} & (\overline{r} \& \overline{s}) \to_L (B \varphi \& B(\varphi \to \psi)) \\
& \vdash_{FB} & \overline{\max(r+s-1, 0)} \to  (B \varphi \& B(\varphi \to \psi)) \\
& \vdash_{FB} & \overline{\max(r+s-1, 0)} \to  B \psi . 
\end{array}$$
\end{proof}

Actually, since necessity measures are a particular kind of belief functions, we can recover deductions in possibilistic logic \citep{DLP94} inside \FBR\ once we extend this logic with the axiom $$ B(\varphi \land \psi) \equiv B\varphi  \land B\psi .$$
In such a case, the following characteristic inference rule of possibilistic logic is derivable: 
$$\{\overline r\to B\varphi, \overline s\to B(\varphi \to \psi)\}\vdash_{FBR}(\overline r \land \overline s)\to B\psi.$$
Note also that \FBR\  includes formulas of the form $\overline r\to \neg B\neg \varphi$, expressing that the plausibility of $\varphi$ is at least $r$, i.e.\ inequalities of the form $pl(\varphi)) \geq r$, where $pl$ is Shafer's plausibility function, the dual function of a belief function. Thus in the above axiomatic extension of \FBR\  it is also possible to capture the so-called {\em generalized possibilistic logic} from  \citep{DPS}.

However, in contrast with possibilistic logic, it is far from obvious that the belief function logic \FBR \ is powerful enough to derive optimal lower bounds $s$ in formulas $\overline s\to B\psi$ inferred from a weighted base of the form $\{\overline r_i\to B\varphi_i, i = 1,\dots,n\}$, using the proof system of the logic \FBR.

%***}

\section{Conclusions} \label{seconc}

We have revisited the belief function logic previously proposed by \citet{GHE03}, based on putting together two ideas:  belief functions on classical propositions can be understood as probabilities on S5 modal formulas, and reasoning about probability can be formalised as a sort of modal theory over {\L}ukasiewicz fuzzy logic. We have shown that this approach can be conceptually and technically simplified by replacing the full S5 language by its subjective fragment with a much simpler semantics based on epistemic states modeled by non-empty subsets of classical propositional interpretations. In short, in the approach of \citet{GHE03} we have replaced S5 by the minimal epistemic logic MEL (capturing Boolean possibility theory).
%in a construction where S5 and {\L}ukasiewicz logics were put together to capture probabilities of modal formulas. 
The various logic systems involved in this paper are described in Table \ref{summary}.

Several points remain to be explored. First,
the belief function logic \FB \ can be specialized so as to recover the {\L}ukasiewicz modal logic accounts of probability and possibility theories. Namely, 

\begin{itemize}

\item To recover the probability logic of \citep{HGE95,H98}, it is enough to add the axiom for the graded probability modality $B(\varphi \lor \psi) \equiv B\varphi \oplus (B\psi \ominus B(\varphi \land \psi))$ to \FB.

\item To recover a form of possibilistic logic where possibility and necessity are graded modalities \citep{HHEGG93,H98}, 
%\citep{HHEGG93},  
it is enough to add the axiom  of necessity measures 
$B(\varphi \land \psi) \equiv B\varphi  \land B\psi$ to \FB.
%
%$N_\alpha \varphi$ corresponds to the formula $\Delta(\overline{\alpha} \to B\varphi)$. 

\end{itemize}

%\begin{table}
%\caption{Links between logical systems in the paper. \\}
%\label{summary}
%\begin{center}
%\red{\begin{tabular}{|c|c|}
%\hline {\bf Components} & {\bf Logic} \\
%\hline S5 + \L + Probability axioms & Probability logic of \cite{HGE95} \\
%\hline S5 + \L + Necessity axioms & Possibilistic logic of \cite{HHEGG93} \\
%\hline S5 + \L + Belief functions axioms  & Belief function logic of \cite{GHE03} \\
%\hline CPL + weights + Necessity axioms & Possibilistic Logic of \cite{DLP94} \\
%\hline $\Box$CPL  + Necessity axioms & MEL (\cite{BD13}) \\
%\hline $\Box$CPL  + weights + Necessity axioms & Generalized Possibilistic Logic (\cite{DPS}) \\
%\hline MEL+ \L + Belief functions axioms & \FB (this paper) \\
%\hline \L + rationals & Rational Pavelka Logic RPL \\
%\hline MEL + RPL + Belief functions axioms & \FBR (this paper) \\
%\hline \end{tabular}}
%\end{center}
%\end{table}

\begin{table}
\begin{center}
\begin{tabular}{|c|c|}
\hline {\bf Components} & {\bf Logic} \\
\hline \L + rationals & Rational Pavelka Logic RPL \citep{H98} \\
\hline CPL + weights + Necessity axioms & Possibilistic Logic of \citet{DLP94} \\
\hline CPL + \L + Probability axioms & Probability logic of \citep{HGE95,H98} \\
\hline CPL + \L + Necessity axioms & Possibilistic logic of \citep{HHEGG93,H98} \\
\hline S5 + \L + Probability axioms  & Belief function logic of \citet{GHE03} \\
\hline $\Box$CPL  + Necessity axioms & MEL \citep{BD13} \\
\hline $\Box$CPL  + weights + Necessity axioms & Generalized Possibilistic Logic \citep{DPS} \\
\hline MEL+ \L + Probability axioms & \FB\ (this paper) \\
\hline MEL + RPL + Probability axioms & \FBR\ (this paper) \\
\hline \end{tabular}
\end{center}
\caption{Links between logical systems in the paper.}
\label{summary}
\end{table}

The belief function logic \FB\ could be extended to a richer underlying language, allowing objective formulas at the inner level. In this spirit, the MEL logic was extended to MEL$^+$  \citep{BDG14,BDGP18}, allowing for the combination of propositional and modal formulas with a semantics slightly modified. Namely formulas in MEL$^+$ are evaluated by pairs $(w, E)$, where the interpretation $w$ (representing the actual world) evaluates objective subformulas,
and the epistemic part $E$ evaluates modal subformulas. It is possible to add a probability logic on top of MEL$^+$. However this would mean that the semantics should be based on probability distributions on $W\times 2^W$, whose practical interpretation is a bit hard to figure out. On the other hand, the natural probabilistic counterpart of pairs $(w, E)$ would be pairs $(p, Bel)$ where $p$ is, say, a frequentist probability function, and $Bel$ a subjective belief function. But it is not clear then how to evaluate the truth of a proposition $P(\Phi)$, where $\Phi$ is a MEL$^+$ formula, using pairs $(p, Bel)$.

Finally, any belief function specialist will notice that a major component of evidence theory  is missing in the \FB\ and \FBR\  logics: Dempster rule of combination. The belief function logic in  \citep{ChatalicDP87} relies on the latter, viewing any logical statement as a basic probability assignment, and combining them using Dempster rule. This combination rule is also captured by the logic in \citep{SossaiBC01}. But none of these two approaches is based on a modal logic. Accounting for Dempster rule in the style of the \FBR\ logic would mean going further from {\L}ukasiewicz or Rational Pavelka logics, so as to encompass a conjunction connective representing the product of rational numbers, as preliminarily done on a particular basis in  \citep{GHE01,GHE03} on top of the whole S5 logic. This is yet another non-trivial open problem.

\blue{
\subsection{Acknowledgments} The authors are grateful to anonymous reviewers for their remarks and suggestions. Godo acknowledges partial support by the MOSAIC project (EU H2020- MSCA-RISE-2020 Project 101007627) and  Spanish project  PID2019-111544GB-C21/AEI/10.13039/501100011033. }

\end{document}